\documentclass[10pt,journal]{IEEEtran}
\usepackage{amsmath}
\usepackage{bm}
\usepackage{graphicx}
\usepackage{float}
\usepackage{amssymb}
\usepackage{amsthm}
\usepackage{authblk}

\newtheorem{theorem}{Theorem}
\newtheorem{corollary}{Corollary}

\begin{document}
%

\title{Benchmarks for Image Classification and Other High-dimensional Pattern Recognition Problems}

\author[*]{Tarun Yellamraju}
\author[$\dagger$]{Jonas Hepp}
\author[*]{Mireille Boutin, member IEEE}
\affil[*]{School of Electrical and Computer Engineering, Purdue University, West Lafayette IN, USA}
\affil[$\dagger$]{Department of Mechanical Engineering, Ilmenau Technical University, Ilmenau, Germany}
\renewcommand\Authands{ and }
\maketitle

\begin{abstract}
A good classification method should yield more accurate results than simple heuristics. 
But there are classification problems, especially high-dimensional ones like the ones based on image/video data, for which simple heuristics can work quite accurately; the structure of the data in such problems is easy to uncover without any sophisticated or computationally expensive method. On the other hand, some problems have a structure that can only be found with sophisticated pattern recognition methods. 
We are interested in quantifying the difficulty of a given high-dimensional pattern recognition problem. We consider the case where the patterns come from two pre-determined classes and where the objects are represented by points in a high-dimensional vector space. However, the framework we propose is extendable to an arbitrarily large number of classes.

We propose classification benchmarks based on simple random projection heuristics. 
Our benchmarks are 2D curves parameterized by the classification error and computational cost of these simple heuristics.
Each curve divides the plane into a ``positive-gain'' and a ``negative-gain'' region. The latter contains methods that are ill-suited for the given classification problem. The former is divided into two by the curve asymptote; methods that lie in the small region under the curve but right of the asymptote merely provide a computational gain but no structural advantage over the random heuristics.

We prove that the curve asymptotes are optimal (i.e. at Bayes error) in some cases, and thus no sophisticated method can provide a structural advantage over the random heuristics. Such classification problems, an example of which we present in our numerical experiments, provide poor ground for testing new pattern classification methods. Our numerical experiments also feature cases where the sophistication of methods like support vector machines or deep neural networks provides structural advantages (left of asymptote) as well as cases where such methods are ill-suited for the given classification problem (in the negative-gain region).

\end{abstract}


\section{Introduction}
In machine learning, 
the two-class pattern recognition paradigm is often formulated in terms of a discriminant function $g(x)$ whose sign should determine the class of the data point $x$ with a high probability of accuracy. The entries of $x\in {\mathbb R}^p$ are features used to represent each object to be categorized, and the function $g(x)$ is estimated numerically using a training set of pre-labeled feature points $x_1,\ldots,x_N \in \mathbb{R}^p$. The function g(x), often concocted using a complicated non-linear combination of the original features $x$, can be viewed as a new feature. It is a distinguishing feature for the classes considered, and many efforts have been spent to develop powerful methods to effectively find good discriminant functions $g(x)$. Deep neural networks are a great example of such \cite{PalmThesis2012}.

Pattern recognition in high-dimension can be quite challenging.
Indeed, finding such a distinguishing feature $g(x)$ can be quite difficult when the feature vectors $x\in {\mathbb R}^p$ are in a space of high-dimension $p$. Unless there are several good choices of features, finding a good $g(x)$ is like looking for a needle in a haystack: without a good trick such as a prior model assumption or other prior information, one needs a considerable amount of numerical work to be successful.

However, recent work suggests that high-dimensional data representing images or videos have a lot of structure, ``so much so that a mere random projection of the data is likely to uncover some of that structure'' \cite{han2015hidden}. There is some evidence that this phenomenon extends to other types of high-dimensional data \cite{BoutinCluster2016}. Thus, it is quite conceivable that there are many high-dimensional pattern recognition problems  where several and easily identifiable good choices of $g(x)$ exist. Such cases are much easier to deal with from a numerical and computational perspective. 


It is not necessarily easy to tell from training data how many good classifiers $g(x)$ exist for a given pattern recognition problem, especially in high-dimension. More generally, there are no good all purpose methods for determining the difficulty of a given pattern recognition problem from training data. However, having such information would be useful, as it could guide the choice of method used to attack such problems in practice.  In the context of machine learning research, that information could also be used to select good datasets for testing new pattern recognition methods: as good results in an easy dataset could give false impressions of success, one should develop new methodologies with the hard datasets in mind and focus the testing efforts on such datasets.

In the following, we provide a framework for quantifying the difficulty of a given pattern recognition problem. Specifically, we propose sequences of upper bounds for the probability of error of a binary classification. The bounds are obtained using simple heuristics based on random projection of the data on a one-dimensional linear subspace; the decisions are made by thresholding the resulting one-dimensional features. If a sophisticated pattern recognition method produces a worse outcome than these bounds, one must conclude that it is ill-suited for that particular pattern recognition problem. In other words, in order to justify its complexity and computational cost, a sophisticated pattern recognition method should yield a significantly smaller probability of error than our proposed error benchmarks for the given classification problem.

The sequences of bounds we propose are all monotonic decreasing (Theorem \ref{thm:mono_k}). 
Meanwhile the computational cost of the underlying classification heuristic starts from extremely modest and then increases gradually. One can use the computational cost of the method (either training cost or testing cost, depending on context) as a proxy for complexity. The trade off between complexity and accuracy for each sequence of bounds can be visualized  by looking at the corresponding curve in the error and computational cost plane, as illustrated in Figure \ref{fig:theoretical_graph}. 
One can look at each of these curves in two ways: 1) for any given desired classification error, each curve provides an upper bound on tolerable complexity (computational cost), and 2) for any given computational cost allowed, each curve provides a maximum tolerable classification error. Thus the curve divides the relevant area of the benchmark plane (above the error axis and right of the vertical line through Bayes Error) into two regions : the ``negative-gain'' region, situated above the curve, and the ``positive-gain'' region (shaded), below the curve and right of the vertical line through Bayes error. Any method lying in the negative-gain region of any benchmark curve (i.e. any value of n) is ill-suited for the classification problem at hand.

Each sequence of bounds we propose converges to a limit; that limit can also be used as an error benchmark which effectively divides the positive-gain region into two: 
the region directly under the curve (right of the asymptote), which we call ``computational-gain'' region, corresponds to methods that only provide a computational advantage over the random heuristics. Indeed, since there exists a random heuristic with comparable error rate (right above the point of the method, on the benchmark curve), the complexity of the structure allowed by the method is not effectively needed to achieve the given error rate. In other words, a comparable or even smaller rate can be achieved with piece-wise linear separations chosen at random, and thus any non-linearity afforded by the method is not effectively exploited. In contrast, we call the region left of the asymptote the ``structural-gain'' region.

In some cases, the limit of our sequences is equal to Bayes Error (Theorem \ref{thm:asymp}). Thus in those cases, the structure-gain region is empty. 
Such pattern recognition problems are extremely easy and do not warrant the use of any sophisticated method. Testing new pattern recognition methods on such datasets should be discouraged. To the contrary, one should look for datasets representing pattern recognition problems for which the structure-gain region is very large. For such datasets, there is a lot of room for improvement in accuracy and thus the use of complex and/or computationally expensive methods is justified.  

\begin{figure}[t]
\centering
\includegraphics[width=0.5\textwidth]{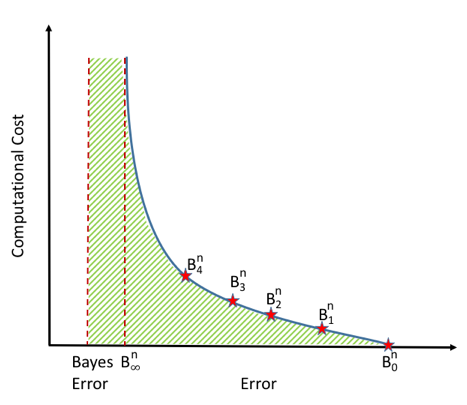}
\caption{\textbf{The Benchmark Plane:} For a given classification problem, the sequence of bounds $B_0^n, B_1^n, B_2^n, B_3^n, \ldots $ for some $n \in \mathbb{N}$ along with their (training) computation time define a curve with asymptote at $B_\infty^n$. Methods whose sophistication is warranted for the problem must lie in the shaded region and to the left of the asymptote.
}
\label{fig:theoretical_graph} 
\end{figure}

\section{Benchmarking Pattern Recognition Problems}
\label{section:related work}
One benchmark traditionally used to put the results of a pattern recognition method in perspective is Bayes Error. Bayes Error is the probability of error of a classifier that assigns classes following Bayes Decision Rule and assuming full knowledge of the class (marginal) probability densities and priors. It is proven to be optimal, in the sense that no classifier can obtain a probability of error less than Bayes Error. Thus, Bayes Error is a benchmark for the probability of error of a classifier which corresponds to the minimum possible probability of error that can be achieved for the given feature vector \cite{duda2012pattern}.

Computing Bayes error from training data is difficult, especially in high-dimension. The task involves estimating the class probability densities and priors, and integrating these functions over potentially complex boundaries. While the case of univariate normally distributed class densities can be computed analytically, high-dimensional distributions (even normal ones) must be handled numerically. Computing Bayes error when the problem does not follow any common probability model is even more difficult. There have been efforts to approximate Bayes error and bound it. For the special case of binary classification where both class densities are multivariate Gaussians, Fukunaga et al \cite{BayesMultiGauss} have presented an explicit mathematical expression that approximates Bayes error. The Chernoff bound \cite{chernoff1952measure} and the Bhattacharyya bound \cite{bhattachayya1943measure} are well-known upper bounds; some closed form expressions for these bounds exist for the special case of Gaussian densities of the underlying classification data, but they are not necessarily tight or particularly insightful for distributions that deviate markedly from a Gaussian \cite{duda2012pattern}. 

Thus in practice, one seldom attempts to either estimate or bound Bayes error. Instead, one merely focuses on developing a classifier with the smallest possible probability of error. The accuracy of the chosen method on some test sets is used to measure the success of that method. For example, the large scale image classifier in \cite{ImageClassificationVanGool} reports classification accuracy compared to other classifiers achieved on a subset of the ILSVRC 2010 dataset \cite{ILSVRC15}. Similarly, the texture classifier in \cite{textureclassifier} reports and compares classification accuracies on three texture datasets : KTH-TIPS2 \cite{KTH-TIPS2}, FMD \cite{FMD} and DTD \cite{DTD}. Yet another example is the DCNN based image classifier in \cite{DCNNClassifier}, which reports classification accuracies on the datasets : MIT-67 \cite{MIT-67}, Birds-200 \cite{Birds-200}, PASCAL-07 \cite{PASCAL-07} and H3D \cite{H3D}. If these test sets are appropriate, then one can conclude that Bayes error for each of these classification problems should be no larger than the error reported with each of the chosen methods.

But as we just mentioned, Bayes Error is merely a minimum bound for the probability of error of a classifier corresponding to the overlap between the probability distributions of the two underlying classes (class overlap): while it may be small for some problems, it is not clear how difficult the underlying pattern recognition problem really is. For example, for the same value of Bayes Error, the optimal class separation could be a simple hyperplane (obvious structure of data that can be easily found) or a very complex hypersurface with highly varying curvatures and multiple connected components (a hidden structure that requires a sophisticated pattern recognition method). Thus the value of Bayes Error provides no insight into the complexity of the classifier needed to solve the problem or the nature of the pattern recognition problem at hand. 

Another benchmark that can be used to judge the performance of a classifier  on a given classification problem is the value of the minimum class prior: $\min \{\text{Prob }(\omega_1),\text{Prob }(\omega_2) \}$, where $\omega_1$ and $\omega_2$ are the two possible classes for the given classification problem. For example, if $\text{Prob }(\omega_1)=98\%$, then building a classifier with a 1.9\% probability of error would not be particularly impressive. Indeed, picking the most likely class ($\omega_1$) regardless of the value of $x$ would nearly beat that classifier. This is relevant even if the class priors are equal $\text{Prob }(\omega_1)=\text{Prob }(\omega_2)=50\%$. Indeed classifiers with an estimated probability of error above $50\%$ are not unheard of. Estimating the class priors from training data is not difficult. So while this benchmark also provides only a little bit of understanding about the nature of the pattern recognition problem at hand, it is useful to consider it when analyzing the results of a classifier.

\section{Random Projections and TARP}
As we mentioned earlier, the problem of building a two-class classifier is often formulated as a quest for a function $g(x):{\mathbb R}^p\rightarrow {\mathbb R}$. The value of $g(x)$ can be viewed as a new feature; decisions are made by thresholding at zero the value of that feature. The difficulty of finding a good $g(x)$ tends to increase with the dimension of the feature vector $x$. Thus some methods first seek to decrease the dimension of $x$. Ideally, this should be done in such a way that the information that distinguishes the classes is preserved. However, this is a chicken and egg problem: it is hard to preserve that information without knowing what it is, but one can hardly know what it is before decreasing the number of dimensions. 

One way to decrease the dimension of $x$ is to project it onto some lower-dimensional subspace. In the simplest cases, the lower-dimensional subspace is linear. Some approaches seek the best projection (in terms of some cost function) by careful numerical optimization. Other approaches project the data in some random fashion. The latter tends to be a lot less computationally extensive than the former. However, one might express concerns about the accuracy of mere random projections. While these concerns are justified, random projection methods are still popular because they have been shown to be surprisingly effective.  
For example, classification algorithms like \cite{two}, \cite{five}, \cite{six} and \cite{seven} use random projections to transform high dimensional data into lower dimensional feature vectors as a pre-processing step before classification. Iterative random projections have also been used in the realm of Big Data to find visual patterns of structure within data by projecting from a high dimensional space to a low dimensional space \cite{four}. An evaluation of the performance of random projections for dimensionality reduction can be found in \cite{three}. 



Previous work  suggests that images \cite{han2015hidden} and other high dimensional data \cite{BoutinCluster2016} have so much hidden structure within that even randomly projecting down to just one dimension will likely uncover some of that structure.  Thus, one can potentially construct a simple classifier based on the value of a feature vector $x\in {\mathbb R}^p$ by generating a random vector $r\in {\mathbb R}^p$ and thresholding the value of the projection  $r\cdot x$. In mathematical terms, we can classify the data according to the rule (after relabeling the classes if needed) \[
r\cdot x\lessgtr^{\omega_1}_{\omega_2}t.
\] 
for some threshold value $t$. Observe that a threshold value of $t=\infty$ would classify all the points into the same class $\omega_1$. Conversely, a threshold value of $t=-\infty$ would classify all the points into the same class $\omega_2$. For optimal accuracy, the threshold $t\in {\mathbb R}$ should be chosen to minimize the probability of error of the classifier.
We call the classifier obtained with the optimal value of the threshold a ``TARP" (Thresholding After Random Projection), for short. 
Observe that, while the accuracy of such a classifier may not be particularly impressive in general, it is no worse than the accuracy of picking the most likely class within the population. In other words, regardless of the projection vector $r$, the probability of error of a TARP is no greater than $\min \{ \text{Prob }(\omega_1),\text{Prob }(\omega_2) \}$. Furthermore, for some values of $r$, it can be quite low depending on the nature of the classification problem at hand. For example, if the data has been extended to a higher dimensional space using a kernel method (as with support vector machines) in such a way that the data became linearly separable, then the probability of error of a TARP classifier could be as low as $0\%$.

\section{TARP-based Benchmarks}\label{section:TARP_def}
Let $B_0$ be the error rate achieved by always choosing the most likely class in the population, in other words 
\[
B_0=\min \{ \text{Prob }(\omega_1),\text{Prob }(\omega_2) \}.
\]
This benchmark is our starting point. Let $r$ be a sample of a p-dimensional random variable ${\mathbf r}\in {\mathbb R}^p$, and consider the projection of the feature vector $x$ onto ${\mathbf r}$, namely $r\cdot x$. Consider the threshold $t$ and the class labeling that minimizes the probability of error of the classifier 
\[
r\cdot x\lessgtr^{\omega_1}_{\omega_2}t.
\]  
Denote by $\varepsilon_1 = \epsilon_1(r)$ the probability of error of this TARP classifier. For an unspecified random vector ${\mathbf r}$, $\epsilon_1$ becomes a random variable ${\pmb \epsilon}_1$ whose probability density function is induced by that of ${\mathbf r}$. Denote by $B_1$ the expected value of  ${\pmb \epsilon}_1$:
\[
B_1=\int_{{\mathbb R}^p} \epsilon_1(r) f_r(r)dr.
\]
The value of $B_1$, which represents the expected error of a TARP classifier under a given probability model for the projection vector ${\bf r}$, is our second benchmark. Observe that $B_1\leq B_0$.

Now consider $n$ samples $r_1,r_2,\ldots,r_n$ of the random vector ${\mathbf r}$ (i.i.d) and the TARP classifier for each projection $r_i\cdot x$, for $i=1,\ldots,n$. Then pick the ``best'' TARP classifier among these $n$ (for example, the one with the smallest probability of error, smallest impurity, lowest entropy, etc.) We call the resulting classifier an ``$n$-TARP'' (best among $n$ TARP classifiers), and denote the probability of error of this classifier by $\varepsilon_1^n = \varepsilon_1^n(r_1,r_2,\ldots,r_n)$. For an unspecified set of random vectors (i.i.d.) ${\bf r}_1,\ldots, {\bf r}_n$, $\varepsilon_1^n$ becomes a random variable ${\pmb \epsilon}_1^n$ whose probability density function is induced by that of ${\bf r}$. Denote by $B_1^n$ the expected value of  ${\pmb \epsilon}_1^n$:
\[
B_1^n=\int_{{\mathbb R}^{p\times n}} \varepsilon_1^n(r_1,\ldots,r_n) f_r(r_1)\ldots f_r(r_n)dr_1\ldots dr_n.
\]
For any positive integer $n$, the value of $B_1^n$ represents the expected error of an $n$-TARP classifier under a given probability model for the projection vector ${\bf r}$. 

Our proposed sequences of benchmarks are built by contructing a $k$-layer binary decision tree using an $n$-TARP at every node of the tree. More specifically, we first fix $n\in {\mathbb N}$ and let $k$ vary among all positive integers. We then consider a $k$-layer tree obtained by generating $n$ random samples of the projection vector ${\bf r}$ for each node and using these projection vector samples to construct a $n$-TARP classifier at each node of the tree. Note that, although our proposed trees are built ``at random," they are not the same as the trees defined by the well-known Random Forests \cite{randomforest}. Indeed, the randomness in our trees comes from the $n$-TARP used for each node, whereas in the case of random forests, it stems from the selection of random subsets of training data to form each decision tree.

We denote the probability of error of a given (sample) decision tree with $k$ levels by $\varepsilon_k^n$. For an unspecified set of random projection vectors (i.i.d.), $\varepsilon_k^n$ becomes a random variable ${\pmb \epsilon}_k^n$ whose probability density function is induced by that of ${\bf r}$. Denote by $B_k^n$ the expected value of  ${\pmb \epsilon}_k^n$, for $k=1,\ldots, \infty$, and set $B_0^n=B_0$, for any positive integer $n$. We propose to use the sequence of bounds $\{B_k^n \}_{k=0}^\infty$ to analyze the nature of the structure of a pattern recognition problem.


\section{Mathematical Properties of Proposed Benchmarks}
\label{theory}

\begin{theorem}[Monotonicity in $k$]
\label{thm:mono_k}
For any integer $n_0 \in \mathbb{N}$, the bounds $\{B_k^{n_0}\}_{k = 1}^{\infty}$ form a monotonic decreasing sequence
\[B_{k+1}^{n_0} \leq B_k^{n_0}\]
converging to a limit
\[B_{\infty}^{n_0} := \lim_{k \to \infty} B_k^{n_0} \]
that is no smaller than Bayes Error
\[B_{\infty}^{n_0} \geq Bayes ~Error. \]
\end{theorem}

\begin{proof}
Let $k \geq 1, ~ k \in \mathbb{Z}$.
Consider a decision tree $T$ with $k+1$ levels built using an $n$-TARP at each node. Observe that $T$ is a random sample of a random tree $\bf{T}$ with a probability density function $\rho_T(T)$ that is induced by the random projection model for ${\mathbf r}$.

Let $\varepsilon_{k+1}^n(T)$ be the probability of classification error of decision tree $T$.
Let $\varepsilon_k^n(T)$ be the probability of classification error of decision tree $T$ restricted to its first $k$ levels.
By construction (since none of the $n$-TARPS considered increases the classification error from the previous level)
\begin{equation}\label{eq:1.samp_ineq}
\epsilon_{k+1}^{n}(T) \leq \epsilon_k^{n}(T).
\end{equation}
Observe that $\varepsilon_{k+1}^n(T)$ and $\varepsilon_k^n(T)$ are samples of the random variables $\pmb{\epsilon}_{k+1}^n$ and $\pmb{\epsilon}_k^n$ respectively.
We have
\begin{equation*} \label{eq:1.work_ineq}
\begin{split}
B_{k+1}^n = \mathbb{E}(\pmb{\epsilon}_{k+1}^n) & = \int_{\mathbb{T}} \epsilon_{k+1}^n(T) \rho_T(T) \,dT \\
 & \leq \int_{\mathbb{T}} \epsilon_{k}^n(T) \rho_T(T) \,dT \\
 & = \mathbb{E}(\pmb{\epsilon}_{k}^n) \\
 & = B_{k}^n
\end{split}
\end{equation*}
where $\mathbb{T}$ is the set of all $k+1$ level decision tress produced by the random projection model. Therefore
\begin{equation}\label{eq:1.4}
B_{k+1}^{n} \leq B_k^{n}~~~ \forall~ k \geq 0,k \in \mathbb{Z}, n \in \mathbb{N},
\end{equation}
and so $\{B_k^{n_0}\}$ is a monotone decreasing sequence in $k$.


By optimality of Bayes Error,
\begin{equation}\label{eq:1.5}
 B_k^{n_0} \geq Bayes~Error ~~\forall~ k \geq 0,k \in \mathbb{Z}, n_0 \in \mathbb{N}.
\end{equation}

From \eqref{eq:1.4}, we know that the sequence $B_k^{n_0}$ is monotone decreasing and from \eqref{eq:1.5}, it is bounded below by the Bayes Error. Hence, using the Monotone Convergence Theorem, the sequence $B_k^{n_0}$ is convergent in $k$ and it converges to its infimum
\[ B_{\infty}^{n_0} = \lim_{k \to \infty} B_k^{n_0}, \quad  B_{\infty}^{n_0} \geq Bayes~Error. \]
\end{proof}

\begin{corollary}
\label{cor:asymp}
The limits $B_{k}^{\infty}$ form a monotonic decreasing sequence
\[B_{k+1}^{\infty} \leq B_{k}^{\infty}\]
converging to a limit
\[B_{\infty}^{\infty} := \lim_{k \to \infty} B_{k}^{\infty} \]
that is no smaller than Bayes Error
\[B_{\infty}^{\infty} \geq Bayes ~Error.\]
\end{corollary}

\begin{theorem}[Optimality of asymptotes]
\label{thm:asymp}
Suppose that the probability density function $f_R (r)$ for the random projection vector ${\mathbf r}\in {\mathbb R}^p$ is such that
\[
p_u = \int_u f_R(r)dr \neq 0,
\]
for any open set $u\in {\mathbb R}^p$.
If the data points $x$ to be classified are drawn from a mixture of two classes $\omega_1$, $\omega_2$ 
\[
\rho(x)=\text{Prob }(\omega_1) \rho(x|\omega_1)+ \text{Prob }(\omega_2) \rho(x|\omega_2)
\]
such that the marginal distributions for both classes are normal with the same covariance matrix 
\[ \rho(x|\omega_1) = \mathcal{N}(\mu_1,\Sigma) ~,~ \rho(x|\omega_2) = \mathcal{N}(\mu_2,\Sigma) ,\]
then all the asymptotes $B_{\infty}^{n}$, $\forall n \in \mathbb{N}$, are optimal
\[ B_{\infty}^n = B_{k}^{\infty} = Bayes ~Error ~~ \forall~ k,n \in \mathbb{N}. \]
\end{theorem}

\begin{proof}
The optimal classifier in the case described is a linear separation hyperplane in ${\mathbb R}^p$. Let $\hat{N}_\rho\in{\mathbb R}^p$ be a unit normal vector to that hyperplane. Observe that if the random vector sample drawn is $r=\hat{N}_\rho$, then the TARP classifier will be optimal. Let $u$ be an open neighborhood of $\hat{N}_\rho$ and let $p_u$ be the probability that ${\mathbf r}\in u$. By assumption, $p_u\neq 0$. Consider $n$ independent random vector samples $\{ r_i \}_{i=1}^n$. The probability that none of the samples lie in $u$ is $(1-p_u)^n$, which approaches zero as $n$ goes to infinity. Thus, with probability one, the infinite random vector sequence $\{ r_i \}_{i=1}^\infty$ contains a vector that is arbitrarily close to $\hat{N}_\rho$. This means that the probability of error of an $n$-TARP (best TARP among $n$ trials) will converge to Bayes Error with probability one as $n$ goes to infinity. Thus the expected value of that error, $B_1^n$, will approach Bayes Error as $n$ goes to infinity. So the limit $B_1^\infty=$ Bayes Error.
By Theorem \ref{thm:mono_k}, $B_{k+1}^n \leq B_k^n$. Taking the limit as $n\rightarrow \infty$, we have  $B_{k+1}^\infty \leq B_k^\infty$ for any $k$, and thus $B_k^\infty=$ Bayes Error, for any $k \in \mathbb{N}$.

Now we look at $B_k^n$. So consider a tree with $k$ levels constructed with a $n$-TARP at each node. Observe that each node at the last level of the tree is concerned with classifying a fraction of the original data space; if we picked $\hat{N}_\rho$ as the random vector at each of these nodes, then the overall $k$-level tree classifier would be optimal. In that case, adding further levels (more $n$-TARP classifiers) below any of these nodes would not decrease the accuracy of the classifier below that node, as it is already optimal. As we travel down each branch of the tree and let the number of levels $k$ go to infinity, the random projection vector sample used for the $n$-TARP at each of the nodes we encounter along our path will form an infinite sequence of vectors. By the same argument as above, that infinite sequence contains a vector that is arbitrarily close to the optimal one $\hat{N}_\rho$, with probability one. Thus the probability of error of a $k$-level tree constructed with an $n$-TARP at each node approaches Bayes Error with probability one, as $k$ approaches infinity. Therefore the expected value of that error also approaches Bayes Error, $B_{\infty}^n =$ Bayes Error.

\end{proof}

\section{$n$-TARP Implementation and Bound Estimation}
We implemented an algorithm to estimate the bounds $B_k^n$ for a binary decision problem using a dataset consisting of labeled (potentially high-dimensional) feature vectors. Our code is available at \cite{HeppCodeDataset}. 
To estimate $B_k^n$, we build a k-level binary decision tree with an $n$-TARP at each node. We split the dataset set into two: 25\% is training data used to select the random vectors for each $n$-TARP,  25\% is cross-validation data used to decide whether to apply a stopping criteria, and 50\% of the data is the testing data used to estimate the error rate of the decision tree. Note that the stopping criteria does not effectively stop the tree from growing, but it prevents the data from being split at that node, so that the tree continues to grow, albeit artificially, to a length of $k$-layers,

For example, in order to estimate $B_1^1$, we construct a 1-level decision tree ($k=1$) with a single 1-TARP. To do this, we generate $n = 1$ random vector(s) $\{r_i\}_{i=1}^n$ of the same dimension as the data, with each element of the vector drawn from a uniform distribution on [-1,1]. As described in Section \ref{section:TARP_def}, we take the inner product of every data point in the training set with the random vector $r_1$. We thus have $n = 1$ set(s) of projections of the training samples. For simplicity, we assume that the projected class distributions for the two underlying classes are 1D Gaussians with mean $\mu_1,\mu_2$ and standard deviations $\sigma_1, \sigma_2$, respectively. Bayes Decision Rule \cite{fukunaga1990statistical} in this case yields up to two thresholds: we pick the threshold $t$ that lies between the two empirical means.
Now, we use the threshold $t$ along with the the random vector $r_1$ to classify the cross-validation data and record the error obtained. If this error is greater than the error observed before classifying the cross-validation set (the prior error in this case), then we choose $t$ to be $\pm \infty$. i.e. we do not split the data into two and record the cross-validation error as the error achieved before classification (the prior error in this case). Having constructed a 1-level decision tree, we use it to classify the testing data and record the testing error. We repeat this construction and testing process 100 times and calculate the average testing error and use it as our estimate of $B_1^1$.

Now to estimate $B_1^n$, we build a 1-level decision tree ($k=1$) with a single $n$-TARP. So we generate $n$ random vectors of the same dimension as the features as before. Projecting the dataset onto each of these random vectors, we obtain $n$ one-dimensional datasets and choose the best threshold using Bayes rule for each of these datasets. Let $t^*$  be the threshold that produces the ``best'' split (we pick the one which decreases the Gini impurity \cite{duda2012pattern} on the training data the most). Following the same steps as above, we cross-validate the tree and compute the testing errors for 100 runs of the algorithm and calculate the average testing error. We declare this as our estimate for $B_1^n$.

For the general case of building a $k$-level decision tree with an $n$-TARP at each node of the tree. Basically, we apply an $n$-TARP at each node of the tree starting at the root. Based on the threshold found at that node, we split the data into two classes and pass one class split to the left child and the other to the right child of the current node. We now recursively apply the $n$-TARP construction process described above at the child nodes formed. At level $k$, we will have $2^k$ nodes and we find the average testing error across the data present at all of these nodes and declare it as our estimate of $B_k^n$. So by building a $k$-level tree in this manner, we can estimate $\{ B_1^n, B_2^n,...,B_k^n\}$. Note that we grow a new set of $k$-level trees and evaluate them for each value of the parameter $k$ in order to estimate $B_k^n$. This means that a different set of trees are grown each time to estimate $B_k^n$ for a specific $k$ and $n$. 

Of course, the results at any particular node in the decision tree are affected by the number of data points present at that node. When building a tree, the number of training samples at each node keeps decreasing as we move down the tree. Naturally, our estimates become unreliable when the number of samples at a node is too small. 
Hence, there is an upper limit on the values that $k$ can take while also providing a good estimate of $B_k^n$. This limit is dependent on the dataset being used.

\begin{figure}[t]
\centering
\includegraphics[width=\linewidth]{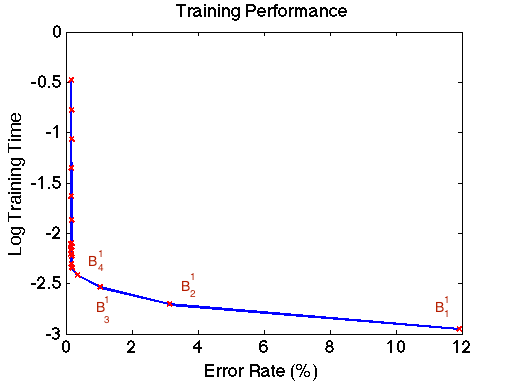} \\	
\caption{\textbf{A Benchmark curve for Normally Distributed Data:} This empirical curve shows the asymptote defined by the limiting error bound $B_\infty^1$. Here $B_\infty^1$ is close to zero, as predicted by Theorem \ref{thm:asymp}.
}
\label{fig:empirical_curves1} 
\vspace{-0.5cm}
\end{figure} 

To test our implementation, we conducted the following experiment. We generated 6000 data samples from a mixture of two Gaussians in ${\mathbb R}^5$ and used 1500 data points for training, 1500 for validation, and the remaining 3000 for testing in the $n$-TARP algorithm. The results for this experiment are presented in Figure \ref{fig:empirical_curves1}. We compute the sequence $B_k^1$ for this data. The parameters for the Gaussian mixture are $\mu_1 = (0,0,0,0,0)$, $\mu_2 = (10,0,0,0,0)$, $\Sigma_1 = \Sigma_2 = {\mathbb I}$. Bayes Error in this case is below machine precision ($10^{-16}$). We observe that the asymptote $B_\infty^1$ is close to zero and equals Bayes Error as expected. This shows that Theorem \ref{thm:asymp} holds true and our bound estimates are accurate.

\section{Experiments with Real Data}
\subsection{Analysis of Digit Recognition Problems}
We first look at two different two-class digit recognition problems. The first problem is that of distinguising between the digit ``0'' and the digit ``1'' on a gray scale image. The other is that of determining whether the digit featured on a gray scale image is even or odd. Both problems will be studied using the MFEAT dataset \cite{MFEAT}, which contains 200 samples of each digit (0-9). The results of our analysis for each of these two problems are presented in Table \ref{table:0v1_bounds} and Table \ref{table:evo_bounds}, respectively. To capture the dependence of the recognition problem on the feature set used to represent the data, we consider three different feature sets: Fourier coefficients (76 features), Karhunen-Love coefficients (64 features) and Zernike Moments (47 features).

We used this data to estimate the first few terms of some of our proposed sequence of bounds. The first bound $B_0$ was estimated as 0.5 since the number of samples of each class for both classification tasks were equal. For the 0-vs-1 classification, the subsequent bounds $B_k^n$ were calculated using the first 100 samples each of 0 and 1 as training data (first 50 of each for training and the remaining 50 for cross validation) and the next 100 samples as testing data for each feature set. Note that the ratio of class 0 and class 1 samples in the training, validation and test sets was the same. For the even-vs-odd classification, the subsequent bounds $B_k^n$ were calculated using the first 100 samples of each digit which were split into even and odd classes as training data (first 50 of each digit for training and the remaining 50 for cross validation) and the next 100 samples as testing data for each feature set. Again, the ratio of class 0 (even) and class 1 (odd) samples in the training, validation and test sets was the same.

As expected, $B_1^n$ and $B_2^n$ for $n = 1,10,50$ are smaller than $B_0$. Note that for the 0-vs-1 classification, $B_1^{50}$ is nearly as low as $2\% $ which is a significant improvement over $B_0$ at $50\% $. Such a small value of the 50-TARP Benchmarks might be surprising, considering their low computational cost. Thus, even without having estimated any of the asymptotes, we can see that the positive-gain regions for each of these feature sets must be very small. On the other hand for the even-vs-odd classification, $B_1^{50}$ is close to $20\% $ for Karhunen-Love coefficients and $30\% $ for the remaining two. The improvement over $B_0$ at $50\% $ is not as great in this case but it is still significant considering it was achieved with naive random projections.

For comparison,  we used MATLAB to build support vector machine (SVM) classifiers  \cite{cortes1995support} using different kernels and parameter values. We also trained a deep neural network classifier \cite{PalmThesis2012} with 2 hidden layers. An Adaboost \cite{boost} classifier was also used for comparison, where the family of weak classifiers used was of the form $p.sgn(x_i - \theta)$ where $p$ is a parity bit, $x_i$ is the $i^{th}$ component of the feature vector $x$ and $\theta$ is a threshold. Training was done over 40 rounds in each case.


\begin{table*}
\begin{footnotesize}
\caption{Classification of images of 0 and 1. The empirical error of support vector machines with various kernels and parameters, Deep Neural Network and Adaboost is compared to the empirical estimate for our proposed error bounds $B_0, B_1^1, B_2^1, B_1^{10}, B_2^{10}, B_1^{50}, B_2^{50}$.}
\begin{tabular}{|c|c|c|c|c|c|c|c|c|c|c|c|}
\hline
Fourier Coefficients &&&&&&&&&&&\\
\hline
Method   &DNN&Adaboost& SVM  & SVM  &  $B_2^1$ & $B_1^1$ & $B_2^{10}$ & $B_1^{10}$ & $ B_2^{50} $ & $ B_1^{50} $ & $B_0$  \\
& &&RBF& Linear& & &  &  & & & \\
parameter values &**& * &$10$ & & & & & & & & \\
Error &0\% &0.5\% & 0\% & 0\% & 16.285\% & 25.125\% & 3.770\% & 6.855\% & 2.205\% & 2.340\% & 50\% \\
Training Time (s)& 9.198 &15.537 & 2.280 & 0.078 & 0.0003 & 0.0006 & 0.0028 & 0.0016 & 0.0102 & 0.0058 & - \\
Testing Time  (s) &0.022 &0.0020 & 0.0455 & 0.0025 & 0.0001 & 0.0001 & 0.00006 & 0.00005 & 0.00006 & 0.00008 & -\\
\hline
\hline
Karhunen-Love Coefficients &&&&&&&&&&&\\ 
\hline
Method   &DNN&Adaboost& SVM  & SVM  &  $B_2^1$ & $B_1^1$ & $B_2^{10}$ & $B_1^{10}$ & $ B_2^{50} $ & $ B_1^{50} $ & $B_0$  \\
& &&RBF& Linear& & &  &  & & & \\
parameter values &**&* &$10$ & & & & & & & & \\
Error &0.5\%&2\% & 1\% & 0.5\% & 22.535\% & 28.775\% & 6.055\% & 10.580\% & 4.315\% & 5.905\% & 50\% \\
Training Time (s)& 8.9281&12.9592 & 0.0494 & 0.0460 & 0.0008 & 0.0003 & 0.0022 & 0.0011 & 0.0086 & 0.0052 & - \\
Testing Time  (s) &0.0247&0.0003 & 0.0052 & 0.0015 & 0.00008 & 0.00009 & 0.00008 & 0.00005 & 0.00011 & 0.00005 & -\\
\hline
\hline
Zernike Moments &&&&&&&&&&&\\ 
\hline
Method   &DNN&Adaboost& SVM  & SVM  &  $B_2^1$ & $B_1^1$ & $B_2^{10}$ & $B_1^{10}$ & $ B_2^{50} $ & $ B_1^{50} $ & $B_0$  \\
& &&RBF& Linear& & &  &  & & & \\
parameter values &**& *&$10$ & & & & & & & & \\
Error &1\%&1.5\% & 1\% & 1\% & 10.925\% & 21.620\% & 2.775\% & 3.865\% & 2.465\% & 2.425\% & 50\% \\
Training Time (s)&7.9931& 9.6483 & 0.0212 & 0.0258 & 0.000002 & 0.0002 & 0.0016 & 0.0013 & 0.0069 & 0.0041 & - \\
Testing Time  (s) &0.0195&0.0003 & 0.0022 & 0.0012 & 0.00011 & 0.00009 & 0.00005 & 0.00002 & 0.00003 & 0.00003 & -\\
\hline
\end{tabular}
\label{table:0v1_bounds}

\end{footnotesize}
$~^*$ The family of weak classifiers used was of the form $p.sgn(x_i - \theta)$ where $p$ is a parity bit, $x_i$ is the $i^{th}$ component of the feature vector $x$ and $\theta$ is a threshold. Training was done over 40 rounds in each case.$~^{**}$The deep neural network consists of 2 hidden layers where the first layer has 35 components and the second layer has 15 components.
\end{table*}

\begin{table*}
\begin{footnotesize}
\caption{Classification of images of even and odd numbers. The empirical error of support vector machines with various kernels and parameters, Deep Neural Network and Adaboost is compared to the empirical estimate for our proposed error bounds $B_0, B_1^1, B_2^1, B_1^{10}, B_2^{10}, B_1^{50}, B_2^{50}$.}
\begin{tabular}{|c|c|c|c|c|c|c|c|c|c|c|c|}
\hline
Fourier Coefficients &&&&&&&&&&&\\
\hline
Method   &DNN&Adaboost& SVM  & SVM  &  $B_2^1$ & $B_1^1$ & $B_2^{10}$ & $B_1^{10}$ & $ B_2^{50} $ & $ B_1^{50} $ & $B_0$  \\
& &&RBF& Linear& & &  &  & & & \\
parameter values &**& *&$10$ & & & & & & & & \\
Error &12.3\%&19.6\% & 15.2\% & Div$^\dagger$ & 38.462\% & 42.854\% & 30.678\% & 35.141\% & 29.047\% & 32.058\% & 50\% \\
Training Time (s)& 14.0915&136.6714 & 0.2155 & - & 0.0025 & 0.0012 & 0.0081 & 0.0047 & 0.0330 & 0.0206 & - \\
Testing Time  (s) &0.0230&0.0016 & 0.0299 & - & 0.0005 & 0.0003 & 0.0005 & 0.0003 & 0.0005 & 0.0003 & -\\
\hline
\hline
Karhunen-Love Coefficients &&&&&&&&&&&\\
\hline
Method   &DNN&Adaboost& SVM  & SVM  &  $B_2^1$ & $B_1^1$ & $B_2^{10}$ & $B_1^{10}$ & $ B_2^{50} $ & $ B_1^{50} $ & $B_0$  \\
& &&RBF& Linear& & &  &  & & & \\
parameter values &**& *&$10$ & & & & & & & & \\
Error &1.4\%&5.9\% & 2.2\% & Div$^\dagger$ & 34.715\% & 40.071\% & 22.378\% & 27.481\% & 16.517\% & 20.687\% & 50\% \\
Training Time (s)&16.4836& 112.4335 & 0.1616 & - & 0.0022 & 0.0011 & 0.0096 & 0.0053 & 0.0289 & 0.0179 & - \\
Testing Time  (s) &0.0238&0.0015 & 0.0201 & - & 0.0005 & 0.0002 & 0.0006 & 0.0003 & 0.0005 & 0.0002 & -\\
\hline
\hline
Zernike Moments &&&&&&&&&&&\\
\hline
Method   &DNN&Adaboost& SVM  & SVM  &  $B_2^1$ & $B_1^1$ & $B_2^{10}$ & $B_1^{10}$ & $ B_2^{50} $ & $ B_1^{50} $ & $B_0$  \\
& &&RBF& Linear& & &  &  & & & \\
parameter values &**&* &$10$ & & & & & & & & \\
Error &15.1\%&23.2\% & 16.3\% & Div$^\dagger$ & 40.056\% & 43.148\% & 34.077\% & 36.210\% & 31.940\% & 33.440\% & 50\% \\
Training Time (s)& 16.9158&82.7901 & 0.3740 & - & 0.0017 & 0.0009 & 0.0065 & 0.0037 & 0.0237 & 0.0149 & - \\
Testing Time  (s) &0.0221&0.0014 & 0.1221 & - & 0.0004 & 0.0002 & 0.0004 & 0.0002 & 0.0004 & 0.0002 & -\\
\hline
\end{tabular}
\label{table:evo_bounds}

\end{footnotesize}
$~^*$ The family of weak classifiers used was of the form $p.sgn(x_i - \theta)$ where $p$ is a parity bit, $x_i$ is the $i^{th}$ component of the feature vector $x$ and $\theta$ is a threshold. Training was done over 40 rounds in each case.$~^{**}$The deep neural network consists of 2 hidden layers where the first layer has 35 components and the second layer has 15 components. $^\dagger$Did not converge after 15000 iterations.
\end{table*}


Observe that, in the case of the 0-vs-1 classification, the  other classifiers are only  slightly more accurate  than the benchmarks despite being significantly more expensive. This is expected since the region of positive gain is very small and thus only modest gains, if at all, can be achieved with more complex methods. There are two factors at play here:  an obvious structure in the data that can be  found by simple heuristics (random projections) as well as a small class overlap  and thus a small value of Bayes Error. Note that this holds for all three feature sets considered. The fact that the structure is obvious makes this particular classification problem not suitable for testing  new  classification methodologies.

The even-vs-odd classification problem is a lot more difficult, as  can be seen by comparing the benchmark values in each case.  The TARP benchmarks go down to about $30\%$ for Fourier Coefficients and the Zernike Moments feature vectors and they go down to about $16\% $ for the Karhunen-Love Coefficients feature vector.  We would need to estimate the value of more terms in each sequence in order to estimate the asymptotes (which we shall do in our further analysis below). However, we can still observe the presence of a method (linear SVM) in the negative-gain region of the benchmark plane. Indeed, it is interesting to note that the method did not converge for any of the feature sets considering that reasonable  piece-wise linear separations can be found at random.  On the other hand, the other methods have a better accuracy than the benchmarks. Thus it is reasonable to imagine that the problem at hand has a hidden structure (i.e., a complex non-linear separation boundary) that cannot be found by simple random projections but that the machinery of an SVM, Adaboost, or DNN can reveal. While the structure is hidden for all feature sets considered, the class overlap appears to vary.  We see a small class overlap  in the case of the Karhunen-Love  coefficients, and  a potentially large class overlap in the case of the Fourier  Coefficients and the Zernike  moments. However the structural-gain region in  each case might actually have a similar size, as the difference between the minimum benchmark values and the smallest of the complex methods accuracy is around 15-20\% in all three cases.  In other words,  all three problems may actually have a structure that is equally well hidden, despite the fact that one potentially has a much smaller Bayes Error than the others.  Such datasets could thus be good candidates for  developing and testing new  classification methods.

\subsection{Analysis of Pedestrian Detection Problems}
We now look at the problem of detecting the presence of a pedestrian on a very low-resolution image. We study this problem using the  Pedestrian Dataset \cite{MunderPedestrian}. This dataset contains low-resolution greyscale images divided into three training sets and two testing sets; we trained with Training Set \#2 (splitting into two for training and validation) and tested on Testing Set \#1. Each set contains 5000 samples without pedestrian and 4800 samples with pedestrian. We used two different feature sets to represent the images \cite{HeppCodeDataset}. 
The first is a feature vector consisting of 648 discrete cosine transform coefficients. The second representation uses 10 rows of the Pascal Triangle of the image \cite{HuangBoutin}. After removing the redundant right-hand-sides of the triangle and storing each complex entry as  two real entries (the middle of the triangle is always real), we ended up with a total of 130 features for each image.  

Again, we computed the first few terms of some of our proposed sequences of bounds. Our results are presented in Table \ref{table:pedestrian_bounds}.
The first bound $B_0$ was estimated as $\frac{4800}{9800}$ (The ratio of the pedestrian samples to the total number of samples). The subsequent bounds $B_k^n$ were estimated using Training Set \#2 and Testing Set \#1 for each set of feature.  Note that the 4900 images we used for training contained the same ratio of no pedestrian (2500) to pedestrian (2400) pictures as the validation set (2500 no pedestrian, 2400 pedestrian).  As expected, $B_1^n$ and $B_2^n$ for $ n = 1,10,50$ are smaller than $B_0$ for both feature sets. In fact, $B_1^{50}$ at $25.3\% $ is nearly half of $B_0$ for the Pascal Triangle coefficient feature vectors.

\begin{table*}
\begin{footnotesize}
\caption{Pedestrian Detection from Low-Resolution Pictures. The empirical error of support vector machines with various kernels and parameters along with Deep Neural Network is compared to the empirical estimate for our proposed error bounds $B_0, B_1^1, B_2^1, B_1^{10}, B_2^{10}, B_1^{50}, B_2^{50}$.}
\begin{tabular}{|c|c|c|c|c|c|c|c|c|c|c|c|}
\hline
With DCT Coefficients &&&&&&&&&&&\\
\hline
Method   &DNN& SVM  & SVM  & SVM  &  $B_2^1$ & $B_1^1$ & $B_2^{10}$ & $B_1^{10}$ & $B_2^{50}$ & $B_1^{50}$ & $B_0$  \\
&&RBF& Linear&MLP &  &  &&&&& \\
parameter values &** &$50$ & &$[1,-1]^*$&&&&&&& \\
Error &22.5\% &21.5\%  &  Div$^\dagger$ & 38.7\% & 44.2\%  & 45.5\% & 37.8\% & 39.4\% & 35.9\% & 36.8\% & 49.0\% \\
Training Time (s)& 577.659& 177.399 &  - & 21.801 & 0.141   & 0.067 & 0.440& 0.268&0.181&0.117  & -  \\
Testing Time  (s) &0.203&  3.123 &  -  & 2.010 & 0.029 &0.022 & 0.029& 0.021&0.028&0.021 & -\\
\hline
\hline
With Pascal Triangles &&&&&&&&&&&\\
\hline
Method   &DNN& SVM  & SVM  & SVM  &  $B_2^1$ & $B_1^1$ & $B_2^{10}$ & $B_1^{10}$ & $B_2^{50}$ & $B_1^{50}$ & $B_0$  \\
&&RBF& Linear&MLP &  &  &&&&& \\
parameter values &**&$50$ & &$[1,-1]^*$&&&&&&& \\
Error& 25.2\% & 27.6\% &  Div$^\dagger$  & 49.50\% & 30.6 \%  & 32.9\% &27.1\% &26.8\% & 25.5\% & 25.3\% & 49.0\% \\
Training Time (s) & 102.987 & 65.560 &  - &  7.061 & 0.032   & 0.017  & 0.103&0.063&0.416&0.280 & -  \\
Testing Time  (s) &0.089 & 2.101 &  - & 0.628 & 0.006 &0.004&0.005&0.004&0.004&0.005 &-\\
\hline
\end{tabular}
\label{table:pedestrian_bounds}

 $~^*$Default parameters in MATLAB. $~^{**}$The deep neural network consists of 2 hidden layers where for DCT, the first layer has 100 components and the second layer has 50 components. For Pascal Triangles, the first layer has 60 components and the second layer has 30 components.$^\dagger$Did not converge after 15000 iterations.
\end{footnotesize}
\end{table*}


For comparison,  we used MATLAB to build support vector machine (SVM) classifiers  \cite{cortes1995support} using different kernels and parameter values. We also trained a deep neural network classifier \cite{PalmThesis2012} with 2 hidden layers. 
Again, the linear SVM did not converge for either of the feature vectors within the default maximum number of iterations (15,000).  Thus, the naive approach of finding a linear separation at random performs better than the iterative approach of the linear SVM in this case. Similarly, the error rate of the SVM with the MLP kernel is higher than some of the bounds listed. Thus both the linear SVM and the MLP kernel SVM lie in the negative-gain region of the benchmark plane. 
On the other hand, for the DCT coefficient feature vector, the error rate of the Deep Neural Network and the SVM with Radial Basis Function Kernel is significantly lower than any of the $B_k^n$ presented in Table \ref{table:pedestrian_bounds}. While we do not have enough data to estimate the position of the asymptotes, we can conjecture that these two methods lie in the structure-gain region of the Benchmark plane. Thus the evidence suggests that the  structure is hidden since simple heuristics do not perform as well as more sophisticated methods.
In contrast, for the case of the Pascal Triangle coefficient feature vector, the bound $B_1^{50}$ is only beat by the Deep Neural Network by a mere $0.1\% $ which is quite small considering the much higher computational cost and sophistication for the DNN. The SVM with the RBF kernel is even worse than our bounds. Thus neither of these methods lie in the positive-gain region of the benchmark plane. It could be that the structure of the problem is obvious and that the overlap between the classes is fairly large (around 25\%). It could also be that the overlap is smaller than that but the structure is very well hidden, so much so that even the DNN or the two SVMs we tried are unable to capture that structure. Thus this dataset could be a good candidate to develop and test new pattern recognition methods. 

\subsection*{\textbf{Estimation of Asymptotes for the Even-vs-Odd Digit Classification Problem}}
Our previous experiments illustrate how to use one, two or a few bounds to analyze the structure and class overlap of a pattern recognition problem. For a more complete analysis, one needs to estimate the asymptotes $B_{\infty}^n$ for some value(s) of $n$.  
To illustrate how to do this, we used the MFEAT data and we focus on the even-vs-odd classification problem and use the Karhunen-Love Coefficient representation.




\begin{figure}[ht]
\centering
\includegraphics[width=0.5\textwidth]{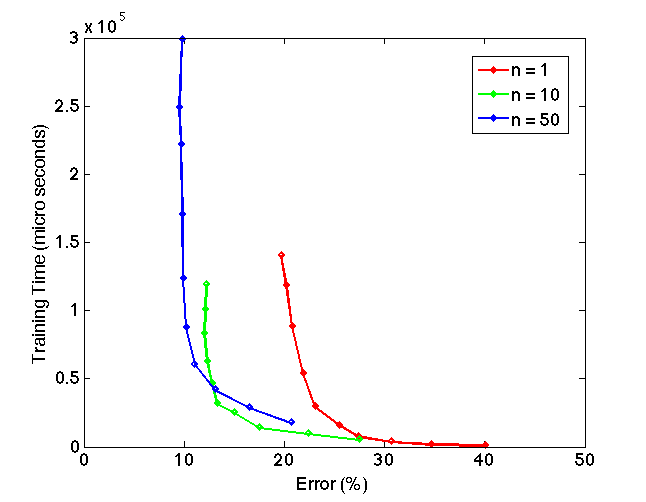}
\caption{\textbf{Benchmark Curves for the even-vs-odd Classification Problem:} The evolution of the sequence of bounds $B_k^n$ for $n = 1,10,50$ using the Karhunen-Love Coefficients clearly shows the location of the asymptote with merely 10 terms.
}
\label{fig:Bkn_evolution} 
\end{figure}



Figure \ref{fig:Bkn_evolution} shows the curves obtained after computing ten terms in the sequence $B_k^n$ for $n$=1,10 and 50. 
This seems to be enough terms to estimate the asymptote for each curve, the smallest of which ($n=50$) appearing to lie around 10\%, still above the error rate of Adaboost (5.9\%), DNN (1.4\%) and the RBF kernel SVM (2.2\%). This provides further evidence that the class separation structure is hidden and that a sophisticated classifier is required. For example, the sophistication of the DNN classifier can decrease the error by a factor greater than seven.

This convergence to a limiting $B_\infty^n$ indicates that our benchmarks are resistant to over-fitting the data as simply increasing the number of levels in the decision tree does not keep reducing the training error to $0\% $. Over the first few levels of the tree, the natural structure present in the data has been found and further levels don't find any new structure that can reduce the error further. The first few levels where the structure within the data is being explored corresponds to the portion of the curves in Figure \ref{fig:Bkn_evolution} where the error is reducing relatively quickly with the level parameter $k$. The later levels where the structure within the data has been completely detected corresponds to the portion of the curves where the errors remain more or less the same for increasing $k$ with increasing computational cost (training time).

\section{Conclusions}

We began by observing that some pattern recognition problems, in particular high-dimensional ones, are a lot easier to solve than others. Indeed the structure of the data is sometimes so easy to find that simple heuristics can lead to a near optimal classification (i.e., with a probability of error close to Bayes error); the existence of such problems was proven in our experiments.  Other problems are a lot more difficult; finding a way to accurately predict the class from a feature vector necessitates a sophisticated method. 

In order to analyze the nature of the structure of the class distributions in a pattern recognition problem, we proposed simple heuristics to obtain upper bounds on the probability of error of a classifier. Our bounds are obtained using extremely low-computation methods based on random projections onto a one-dimensional subspace of the feature space and are particularly well-suited to analyze high-dimensional data sets. We use these bounds to construct a sequence of benchmark curves parameterized by probability of error and computational cost. Each curve determines an error asymptote which we proved to be optimal (equal to Bayes Error) in some cases. When using a non-trivial pattern recognition method on a specific classification problem, one should check that the computational time and probability of error of the method are situated on the left-hand-side of our proposed curves (i.e., in the positive-gain region of the benchmark plane); else their sophistication is either unwarranted (on or near the benchmark curve) or unsuited (right of the curve, the negative-gain region of the benchmark plane) for the structure of the data.

To illustrate our proposed benchmarking framework, we looked at two types of digit recognition problems: the problem of distinguishing between ``0'' and ``1'' on a gray-scale image of a hand-drawn digit, and the problem of  distinguishing between even and odd integers on a gray-scale image of a hand-drawn digit. Our analysis indicates clearly that the structure of the first problem is a lot easier to find than that of the second. Such simple classification problems should not be used for testing and developing new pattern recognition methods. We also looked at two pedestrian detection problems. For one of these problems, none of the popular pattern recognition methods we tried was found to lie in the positive-gain region of the benchmark plane. Thus the evidence suggests that the problem at hand has an obvious structure that can be found by random projection, while the class overlap is fairly large (Bayes Error near 25\%). Based on evidence presented in \cite{han2015hidden}, we expect many pattern recognition problems based on image/video data to have a similarly obvious structure.


\bibliographystyle{IEEEtran}
\bibliography{refs_tarun}

\end{document}